\documentclass{article}

\usepackage[preprint, nonatbib]{neurips_2020}

\usepackage[utf8]{inputenc} 
\usepackage[T1]{fontenc}    
\usepackage{hyperref, xcolor}       
\usepackage{url}            
\usepackage{booktabs}       
\usepackage{amsfonts}       
\usepackage{nicefrac}       
\usepackage{microtype}      
\usepackage{nameref}
\usepackage{wrapfig}

\usepackage{amsthm}
\newtheorem{prop}{Proposition}

\newtheorem{lemma}{Lemma}

\theoremstyle{definition}
\newtheorem{defn}{Definition}
\newtheorem{cor}{Corollary}
\definecolor{urlcolor}{rgb}{0,0.2,0.6}
\hypersetup{colorlinks=true, urlcolor=urlcolor, linkcolor=urlcolor, citecolor=urlcolor}

\usepackage{graphicx, float}
\usepackage{booktabs}
\usepackage{tikz}
\usepackage{pgfplots}
\usepackage{color,soul}
\usepackage{amsmath,amsfonts,amssymb,amsthm}
\usepackage{empheq,boldline,textcomp,gensymb, lipsum}
\usepackage{caption}
\usepackage{subcaption}

\newcommand{\R}{\mathbb{R}} 
\DeclarePairedDelimiter\abs{\lvert}{\rvert} \DeclarePairedDelimiter\norm{\lVert}{\rVert}
\makeatletter
\let\oldabs\abs
\def\abs{\@ifstar{\oldabs}{\oldabs*}} 
\let\oldnorm\norm
\def\norm{\@ifstar{\oldnorm}{\oldnorm*}} 
\newcommand\myeq{\stackrel{\mathclap{\normalfont\mbox{def}}}{=}}
\makeatother

\newcommand{\breaktowidth}[2]{\vtop{\hsize#1\noindent$#2$}}

\pgfplotsset{compat=1.14}

\usepackage{multibib}
\newcites{New}{Supplementary References}

\begin{document}

\title{Advantages of biologically-inspired adaptive neural activation in RNNs during learning}

\author{%
Victor Geadah 
    \thanks{Mila - Quebec Artificial Intelligence Institute, Canada}~ 
    \thanks{Université de Montréal, Département de Mathématiques et Statistiques, Canada} 
\And
Giancarlo Kerg
    \footnotemark[1]~
    \thanks{Université de Montréal, Département d'Informatique et Recherche Opérationelle, Canada \newline}
\And 
Stefan Horoi
    \footnotemark[1]~ \footnotemark[2]
\AND 
Guy Wolf
    \footnotemark[1]~ \footnotemark[2]
\And 
Guillaume Lajoie
    \footnotemark[1]~ \footnotemark[2]~ 
    \thanks{Correspondence to: <g.lajoie@umontreal.ca>}
}

\newpage

\maketitle
\setcounter{footnote}{0}

\begin{abstract}

Dynamic adaptation in single-neuron response plays a fundamental role in neural coding in biological neural networks. Yet, most neural activation functions used in artificial networks are fixed and mostly considered as an inconsequential architecture choice.
In this paper, we investigate nonlinear activation function {\it adaptation} over the large time scale of learning, and outline its impact on sequential processing in recurrent neural networks.
We introduce a novel parametric family of nonlinear activation functions, inspired by input-frequency response curves of biological neurons,  which allows interpolation between well-known activation functions such as ReLU and sigmoid.
Using simple numerical experiments and tools from dynamical systems and information theory, we study the role of neural activation features in learning dynamics. 
We find that activation adaptation provides distinct task-specific solutions and in some cases, improves both learning speed and performance.
Importantly, we find that optimal activation features emerging from our parametric family are considerably different from typical functions used in the literature, suggesting that exploiting the gap between these usual configurations can help learning.
Finally, we outline situations where neural activation adaptation alone may help mitigate changes in input statistics in a given task, suggesting mechanisms for transfer learning optimization.
\end{abstract}

\section{Introduction}

Beyond synaptic plasticity that alters the strength of connection between neurons in the brain, there is a set of physiological mechanisms that drive individual neurons to change their input-output properties over several timescales. Collectively, these mechanisms are referred to as {\it neural adaptation}, and they are known to considerably influence the neural code~\cite{Gjorgjieva:2016bt,adaptation}.
In contrast, single neuron response in artificial neural networks (ANN) is often chosen to be a fixed property, both in time and across neurons, as the choice of activation function or ``nonlinearity'' is often made from a limited set of commonly used functions.
Nevertheless, recent practices where the leak rate of leaky rectified linear units (ReLUs) are included as optimized parameters during ANN training (e.g.~\cite{Vorontsov:2017tt}) suggest that allowing a form of ``activation adaptation" alongside learning can be beneficial. Still, what level of flexibility should be allowed, and how neural response modulation impacts learning in ANNs, are far from understood.

In this paper, we draw inspiration from biophysical mechanisms and investigate the impact of activation function adaptation, across training epochs and neurons, on learning dynamics of ANNs. Our goal with this work is not to improve on state-of-the-art performance on precise tasks, but rather to offer a mechanistic exploration of neural response flexibility in ANNs, and identify opportunities for ANN improvement.
Our contributions are threefold. We first propose in \S\ref{ss:gamma_formulation} a novel two-parameter, differentiable family of activation functions where both gain and saturation levels can be modulated. This family contains commonly used activation functions such as ReLU and sigmoid, and can continuously interpolate in between them. Second in \S\ref{ss:signal_measurements}, we analyze the impact of parametric activation changes on the dynamics of random recurrent neural networks (RNN), thereby establishing basic stability and information propagation properties linked to neural activation. Lastly in \S\ref{s:results}, through a series of simple numerical experiments with sequential processing tasks on RNNs, we investigate the impact of adaptation on learning by allowing levels of gain and saturation to evolve throughout training. 
We find that adaptation during training can lead to improvements in learning speed and performance, and gives rise to task-specific configurations reflecting dynamic computational requirements. Moreover, we showcase promising results in \S\ref{ss:results:transfer} that suggest activation adaptation alone can help a network adapt to new tasks, thereby facilitating a form of transfer learning. 
Finally, we discuss how our proposed framework may help understand some forms of neural adaptation in the brain.

\section{Background}
Historically, the choice of activation function in ANNs has ranged from the binary heaviside or ``step'' function, to the differentiable but saturating sigmoid and hyperbolic tangent, to the piece-wise linear (leaky) rectified linear unit (ReLU)~\cite{Hahnloser:2000wl}. Recently, the deep learning field largely adopted the ReLU as it makes backpropagation more efficient on modern hardware, with little empirical impact on model performance compared to other architecture choices. Still, beyond the small list of functions mentioned above, there is little work aimed at understanding activation modulation in modern ANNs.
Recent work aims to bridge this gap~\cite{Vecoven:2020ba}, bringing ideas of neuromodulation in the brain to deep learning, and our contribution builds on this direction with a focus on exploring adaptive activations during learning.

In contrast to ANNs, considerable computational work was dedicated to understand how neural adaptation and network parameter modulation influence coding in the brain (see~\cite{adaptation,Marder:2014du} for reviews). 
Of importance to our study is the neuron's input response function, known as the input frequency curve (IF-curve)~\cite{gerstner_kistler_naud_paninski_2014}. This measures the spiking frequency of a neuron as a function of applied input current, and is what inspired nonlinear activation functions used in ANNs today (see e.g.~\cite{Hahnloser:2000wl}). Recent findings highlight how critical the shape of IF-curves is in both the dynamical behavior of neural networks \cite{SingleNeuronNonlins, Crisanti_2018}) and how populations of neurons in the brain encode and process information \cite{hennequin, PNASnetworks}.
Crucially, neuron IF properties are known to adapt to changes in stimuli statistics, a property that has been shown to have important implications on the neural code~\cite{Fairhall:2001fx,Lundstrom:2008km}.
Other types of IF adaptation occur over developmental timescales, with likely implications on learning and early synaptic and network configurations~\cite{Mease:2013hd}. It is this latter flavor of neural adaptation that inspired the present work, as we set out to understand the role of activation modulation on the order of training epochs in ANNs.

\section{Model and analysis details}\label{s:adaptation_theory}

\subsection{Parametric family of activation functions}\label{ss:gamma_formulation}

We propose a novel, differentiable family of activation functions defined by
\begin{align}
    \gamma (x ;n,s) &= (1-s)\frac{\log (1+ e^{n x })}{n} + s\frac{e^{nx }}{1+e^{nx}} \label{eq:phi}
\end{align}
for $x \in \R$ with two parameters controlling its shape: the \textit{degree of saturation} $s$ and \textit{neuronal gain} $n$\footnote{We often shorten these to \textit{saturation} and \textit{gain} and collectively refer to them as the \textit{shape parameters}}.
This is a $s$-modulated convex sum of two $C^{\infty}(\R)$ functions: the non-saturating softplus ($s=0$), and the saturating sigmoid ($s=1$), while $n$ rescales the domain and controls response sharpness, or gain~\cite{ChaosInNNs}.
Figure \ref{fig:Panel_1:shape_MLE_norm}A-C shows the graph of $\gamma$ for different values of $(n,s)$, interpolating between well-known activation functions in deep learning. 
\begin{figure}[t]
    \centering
    \includegraphics[width = \textwidth]{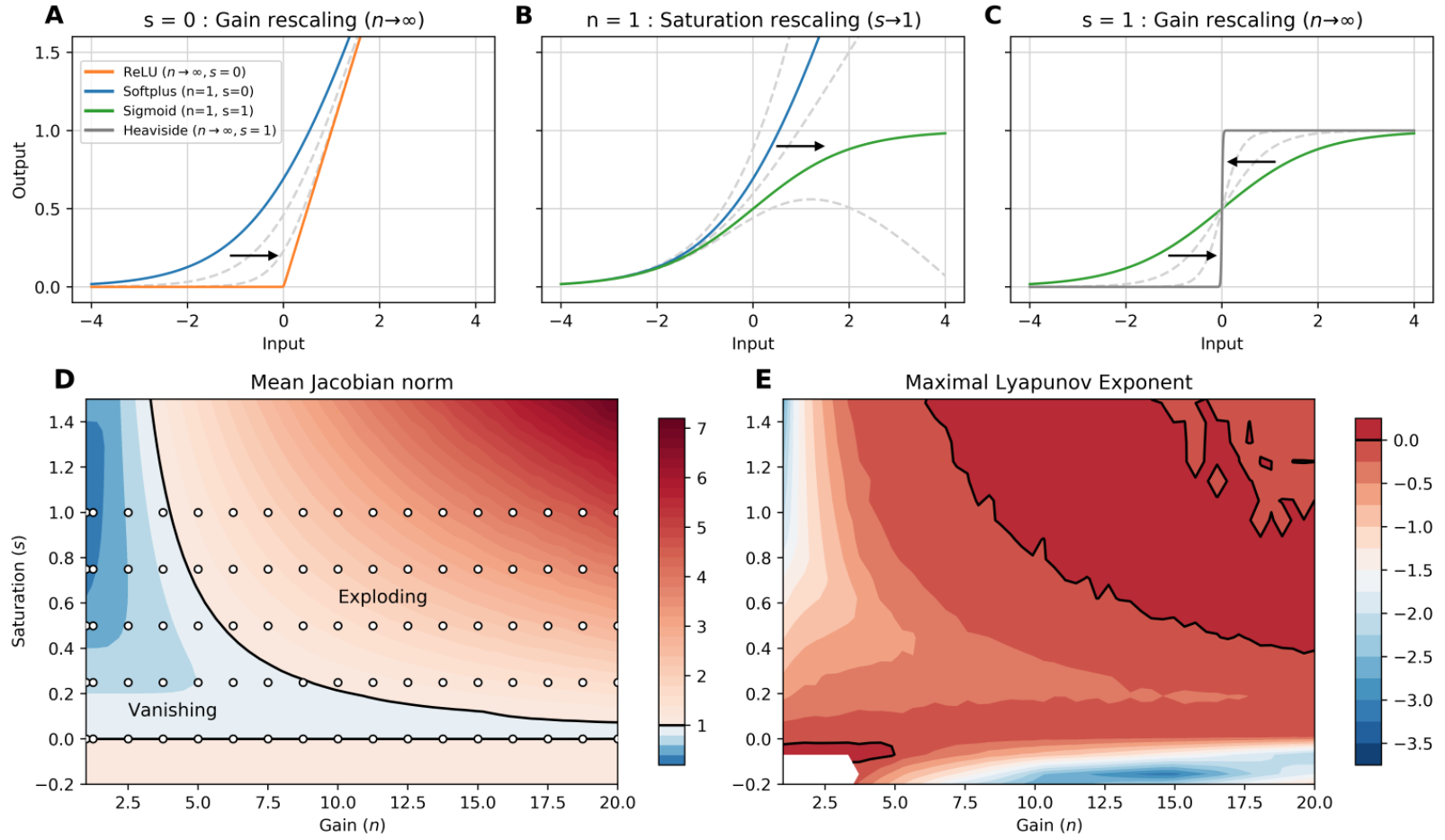}
    \caption{\textit{Activation function $\gamma$ and its properties}. \textbf{A}-~\textbf{C.} various $\gamma$ shape rescaling. Legend inlet in \textbf{A} applies to the first row and indicates well-known activation functions\protect\footnotemark~ attained by $\gamma$. \textbf{D},~\textbf{E.}. Contour plots of respectively the Jacobian norm and maximal Lyapunov exponents in parameter space, computed with random weights. Points in \textbf{D} represent the initialization grid used in this work.}
    \label{fig:Panel_1:shape_MLE_norm}
\end{figure}
We note $\gamma$ is differentiable in both $s$ and $n$, and include these parameters in the optimization scheme in several experiments described below. We refer the reader to Appendix \S\ref{appendix:task_setup} for error gradient derivations that include these parameters.

\subsection{Recurrent neural network model} In line with the goal of isolating the role of activation functions, we elect to use a simple ``vanilla'' RNN model for experiments. The vector equation for the recurrent unit activation $\boldsymbol{h}_t \in \R^{N_{rec}}$ in response to input $\boldsymbol{x}_t \in \R^{N_{in}}$, $t \in \{0,\dots,T\}$ is given by
\begin{equation}
    \boldsymbol{h}_t = \gamma  (W_{rec} \boldsymbol{h}_{t-1} + W_{in}\boldsymbol{x}_t + \boldsymbol{b}; \boldsymbol{n}, \boldsymbol{s}) \label{eq:hiddenstates}
\end{equation}
where the output $\boldsymbol{y}_t \in \R^{N_{out}}$ is generated by a linear readout $\boldsymbol{y}_t = W_{out}\boldsymbol{h}_t + \boldsymbol{b}_{out}$. Weight matrices $W_{(\cdot)}$ and biases $b_{(\cdot)}$ are optimized in all experiments. In results presented below, we explore different ways in which learning is influenced by the shape of $\gamma$, which operates point-wise on its inputs. We consider three scenarios:
\begin{enumerate}
    \item \ul{Static activations} : We explore a range of $(n,s)$ values, but they remained fixed throughout training and are shared by all neurons.
    \item \ul{Homogeneous adaptation} : $(n,s)$ are shared by all neurons, and this parameter tuple is included in the optimization process.
    \item \ul{Heterogeneous adaptation} : $(n,s)$ hold different values across neurons ($\boldsymbol{n},\boldsymbol{s} \in \R^{N}$). We write $\gamma(\boldsymbol{x};\boldsymbol{n},\boldsymbol{s})_i = \gamma(x_i;n_i,s_i) \in \R$ and include all $(n_i, s_i)$ in the optimization process.
\end{enumerate}
The vector of all trainable parameters is denoted $\Theta$, and are updated via gradient descent using backpropagation through time (BPTT), with the matrix $W_{rec}$ initialized using a random orthogonal scheme \cite{henaff} (details in Appendix $\S$\ref{appendix:task_setup}).

\protect\footnotetext{Note that while plotted for finite $n$, ReLU and Heaviside correspond to limits of $\gamma$ as $n \rightarrow \infty$.}

\subsection{Signal propagation measurements}\label{ss:signal_measurements}

Before exploring the impact of our activation function's shape on learning, we first explore how $(n,s)$ shape network dynamics, where connectivity weights are chosen randomly like at initialization, and when the network does not receive inputs. 
Dynamics in this regime are important as they dictate gradient stability properties early in training. Hence to assess how the activation gain and saturation influence on signal propagation, we use three quantities: 

\paragraph{Jacobian norm} The main mechanism leading to the well studied problem of exploding \& vanishing gradients in backpropagation and BPTT happens when products of Jacobian matrices explode or vanish~\cite{trainingRNNs,Bengio:1994do}. We average the $L^2$-norm of the derivative of Eq.~\eqref{eq:hiddenstates} with respect to $\boldsymbol{h_{t-1}}\sim \mathcal{U}(-5,5)$. A mean \ul{Jacobian norm (JN)} that is greater/less than one leads to exploding/vanishing gradients, respectively. An issue with this approximation is that the true mean depends on $\boldsymbol{h_t}$'s invariant distribution, which changes with $(n,s)$. 

\paragraph{Lyapunov Exponents} Developed in Dynamical Systems theory, Lyapunov exponents measure the exponential rate of expansion/contraction of state space along iterates (see Appendix \S\ref{LE_intro} for a LE primer). As an asymptotic quantity, the \ul{Maximal Lyapunov exponent (MLE)} has been used to quantify ANN stability and expressivity~\cite{Pennington:2018tg,BenPoole:2016vya}. We numerically compute it for our system. The MLE gives a better measurement of stability than the Jacobian norm above, although it requires more effort to approximate. A positive MLE indicates chaotic dynamics and can lead to exploding gradients, while a negative MLE leads to vanishing ones.

\paragraph{Mutual Information} We compute the \ul{mutual information (MI)} between inputs and network hidden states, both treated as random variable $X_t$ and $H_t$  (see Appendix \S\ref{MI_intro} for a primer). Borrowing from the information bottleneck principle~\cite{Tishby:1999;IB,Tishby:2017}, lower MI between these variables indicates more computations produced by the network.

Figure~\ref{fig:Panel_1:shape_MLE_norm}D,~E shows the task-independent stability metrics of JN and MLE for a range of ($n$,$s$) values (fixed across neurons). Clearly, activation shape influences Jacobian norms and will play an important role during training. Consistent with the average gradient norm, the MLE reports distinct ($n$, $s$)-regions of stability for random networks. In some cases, expansion and contractions can be useful for computations, and we further use these measurements to study training dynamics.

\section{Results}\label{s:results}

\begin{table}[t]
    \centering
    \begin{tabular}{c  c c c c}
        \toprule
         & \textbf{Copy} & \textbf{psMNIST} &  \multicolumn{2}{c}{\textbf{Character-level PTB}}\\
          & Accuracy (\%) & Accuracy (\%)  & BPC  & Accuracy (\%)  \\
         \midrule
        Specific activations & & & & \\
         \midrule
         Sigmoid & 15 & 62.2 & 1.602 & 65.8 \\
         Softplus  & 47 & 60.4 & 1.540 & 67.1 \\
         ReLU & 100 & 90.1 & 2.865 & 40.4 \\
         \midrule
         Adaptation scenario & & & & \\
         \midrule
         Static  & 100 & 95.0 & 1.534 &  67.4 \\
         Homogeneous & 100 & \textbf{95.2} & 1.529 & 67.4 \\
         Heterogeneous &  100 & 94.8 &\textbf{1.518} &  \textbf{67.6}\\
         \bottomrule
    \end{tabular}
    \vspace{0.2cm}
    \caption{(\textbf{Top}) Test performance for specific activations attained\protect\footnotemark~by $\gamma$, trained without adaptation. (\textbf{Bottom}) Optimal performance on test sets for the different adaptation scenarios.}
    \label{tab:performance}
\end{table}
\footnotetext{See legend in Figure \ref{fig:Panel_1:shape_MLE_norm}A for the corresponding $(n,s)$ combinations.}

We seek to understand how allowing $(n,s)$ to adapt during training (i.e.~be optimized) influences sequential processing in RNNs. Specifically, we are interested in the interplay between long-term and short-term processing, and how distinct activation configurations may facilitate one over the other. To this end, we conduct experiments on three synthetic tasks that demand distinct types of computations (see~\cite{nnRNN} for a discussion):
\begin{enumerate}
    \item \ul{Copy} : Introduced in \cite{LSTM}, the task requires propagation of information over long time scales. The model reads a sequence of inputs, waits for a time delay $T$ (we use $T = 200$) before outputting the same sequence.
    
    \item \ul{psMNIST} : This task requires an accumulation of information over long timescales. A fixed random permutation is applied to pixels of the popular hand-written digits MNIST dataset, and the model reads them sequentially. Correct digit class needs to be outputed at the end.
    
    \item  \ul{PTB} : This task requires ongoing information processing and inference at each time step. Character-level prediction with the Penn Tree Bank (PTB) corpus consists of predicting the next character in a sequence of text~\cite{PTB}. Performance on this task is compared in terms of test mean bits per character (BPC), where lower BPC is better, and accuracy.
\end{enumerate}
We use PyTorch \cite{PyTorch} and the Adam optimizer \cite{Adam} and refer to the Appendix $\S$\ref{appendix:task_setup} for task setups, training details and sensitivity analysis. 

\subsection{Adaptation scenarios and their task-dependent impact on performance}

In this section, we investigate the impact of our parametric family of activation functions on performance within the distinct adaptation scenarios. 
We refer to Table \ref{tab:performance} and Figure \ref{fig:main_performance_fig} for a comparison of the final performances for the various scenarios, and to Appendix \S\ref{appendix:p_results} for further results.

\begin{figure}[t]
    \centering
    \includegraphics[width=\textwidth]{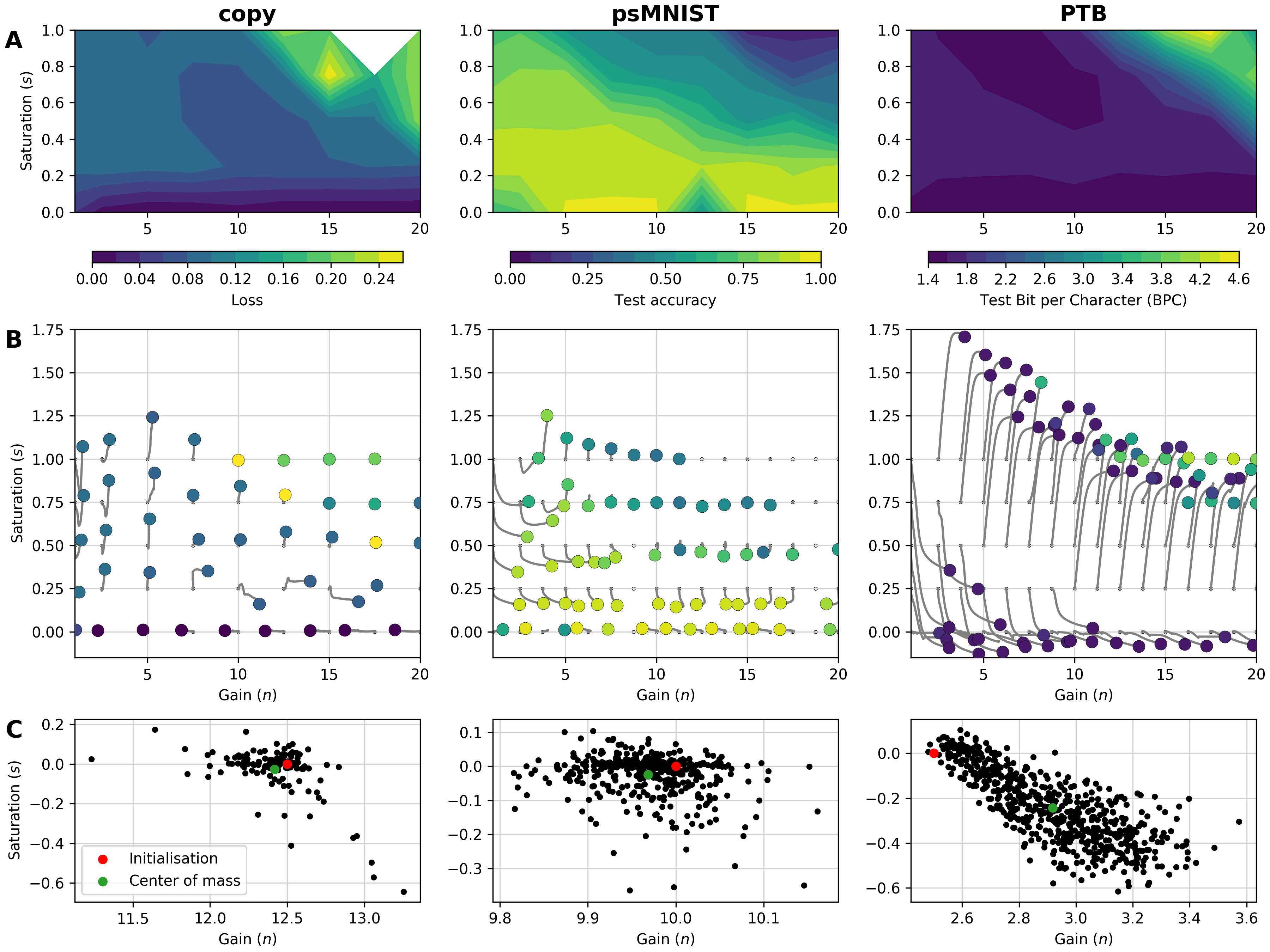}
    \caption{\textbf{A.} Final static performance, where for each point $(n,s)$, we report the performance attained by the model with the corresponding activation function. \textbf{B.} Adaptation dynamics in parameter space. Points are colored by final performance, following the colorbar above. Lines indicate shape parameter trajectories over training. \textbf{C.} Learned heterogeneous $(n,s)$ parameters for the best performing model.}
    \label{fig:main_performance_fig}
\end{figure}

\subsubsection{Static scenario: advantages of static but nontrivial activation functions} A recurrent neural network's ability to accomplish a given task dependents entirely on its ability to efficiently propagate information back through the multiple internal transformations, in the process of backpropagation. As highlighted in \cite{trainingRNNs}, recurrent neural networks are subject to the vanishing and exploding gradient problem, which we explain and highlight the associated expected impact on performance in \S\ref{ss:signal_measurements}.

As expected, we find a strong correlation between the norm of the Jacobian in parameter space which is task-independent (Figure\ref{fig:Panel_1:shape_MLE_norm}D) and the performance landscapes for each task (see Figure \ref{fig:main_performance_fig}A).
Interestingly, regions in space $ (n, s) $ with poor performance are all associated with an exploding gradient, not a vanishing gradient.
Networks whose activation functions have shape parameters in a neighborhood of  $\{(n,s) : \norm{\gamma'(x;n,s)} = 1\}$ present optimal performance, on all the tasks.
On the one hand, this further emphasizes the performance of ReLU (see \cite{ReLU_Glorot_Bengio}) as part of this $(n,s)$-neigborhood. %
However, as we show in Table \ref{tab:performance}, traditional nonlinearities (including ReLU) are outperformed by the considerably different activation functions arising in the different adaptation scenarios. 
This result highlights that non trivial combinations of parameters may also achieve optimal performance while allowing for much more complex nonlinear transformations than ReLU.

\subsubsection{Homogeneous adaptation improves learning speed and overall performance} 

Firstly, regardless of the adaptation scenario considered, we observe that only models with $(n,s)$ parameters located on the $s=0$ axis reach an accuracy of 1.0 for the Copy task (see Figure \ref{fig:main_performance_fig}B-right). This was to be expected since there is an explicit and optimal solution for this task simply involving rotation matrices and linear activation functions~\cite{henaff}. The activation functions closest to linear\footnote{i.e that minimize the operator distance between them and an arbitrary linear function} are precisely combinations of parameters with $s=0$. However, focusing on these parameters, we observe that the gain levels in the presence of adaptation \emph{decrease} through learning, in opposition to adaptation dynamics that would converge towards and encourage the use of ReLU.

For psMNIST and PTB, the introduction of homogeneous adaptation is useful both in terms of learning speed and overall test performance.
Throughout training for both tasks, homogeneous adaptation outperforms the static scenario. There is an increase in final performance (see Table \ref{tab:performance}), and more importantly on learning speed : modal values for homogeneous adaptation reach within 5\% of optimal performance 85\% faster than static for psMNIST, and 29\% faster than static for PTB (see Appendix \S\ref{appendix:p_results}, Figure \ref{fig:performance_epochs} for the performance over training). Moreover, learned combinations of $(n,s)$ parameters differ from the conventional activations (see Figure \ref{fig:main_performance_fig}B) but do so differently for each task. Shape parameters converge \textit{between} known functions for psMNIST, and conversely converge towards previously unexplored $(n,s)$-regions for PTB.

\subsubsection{Heterogeneous adaptation} The heterogeneous adaptation scenario provides a more in-depth portrait of adaptation in RNNs, in a way that is reminiscent of the diversity of activations in cortical networks. We plot in Figure \ref{fig:main_performance_fig}C the final state of learned $(n_i,s_i)$ combinations for the best performing network for each task. 

Allowing for heterogenous adaptation has a different impact on each task. First, the Copy task is the only task for which the use of heterogeneous adaptation was not beneficial, although this scenario also reaches an accuracy of 1.0 like the others. This may be due to the fact that in the presence of a heterogeneous adaptation, certain combinations $(n_i,s_i)$ scatter away from the $s=0$ axis (see Figure \ref{fig:main_performance_fig}C-left), while they remain on it with homogeneous adaptation. For psMNIST, heterogeneous adaptation outperforms the static scenario but not necessarily the homogeneous scenario. Finally heterogeneous adaptation is beneficial for the PTB task as it outperforms all other scenarios in terms of final performance, and matches homogeneous adaptation in terms of learning speed. Interestingly this spectrum of impact of heterogeneous adaptation on performance may be related to a given task need for computation across different timescales. We further investigate heterogeneous adaptation in \S\ref{ss:results:transfer} suggesting its usefulness in transfer learning experiments.

\subsection{Signal propagation is influenced by activation function shape}
 
Recent work \cite{nnRNN} investigates the interplay between long-term and short-term processing, and highlights a spectrum of connectivity structure using the Schur decomposition of learned weight matrices in RNNs. At one end, we find tasks that are highly dependent on hidden-states dynamics and on long-term memory (copy), and at the other end, tasks that give more importance to encoding efficiently and short-term computation (PTB). The psMNIST task lies in the middle. 
In this section, we analyze the impact of the activation function on signal propagation, both from a stability perspective with MLEs and from an information theoretic perspective with mutual information. These measurements complement each other and reflect the different computational properties of the tasks. In our analysis, we suggest mechanisms that lead to the observed adaptation dynamics.
\begin{figure}[t]
    \centering
    \includegraphics[width=\textwidth]{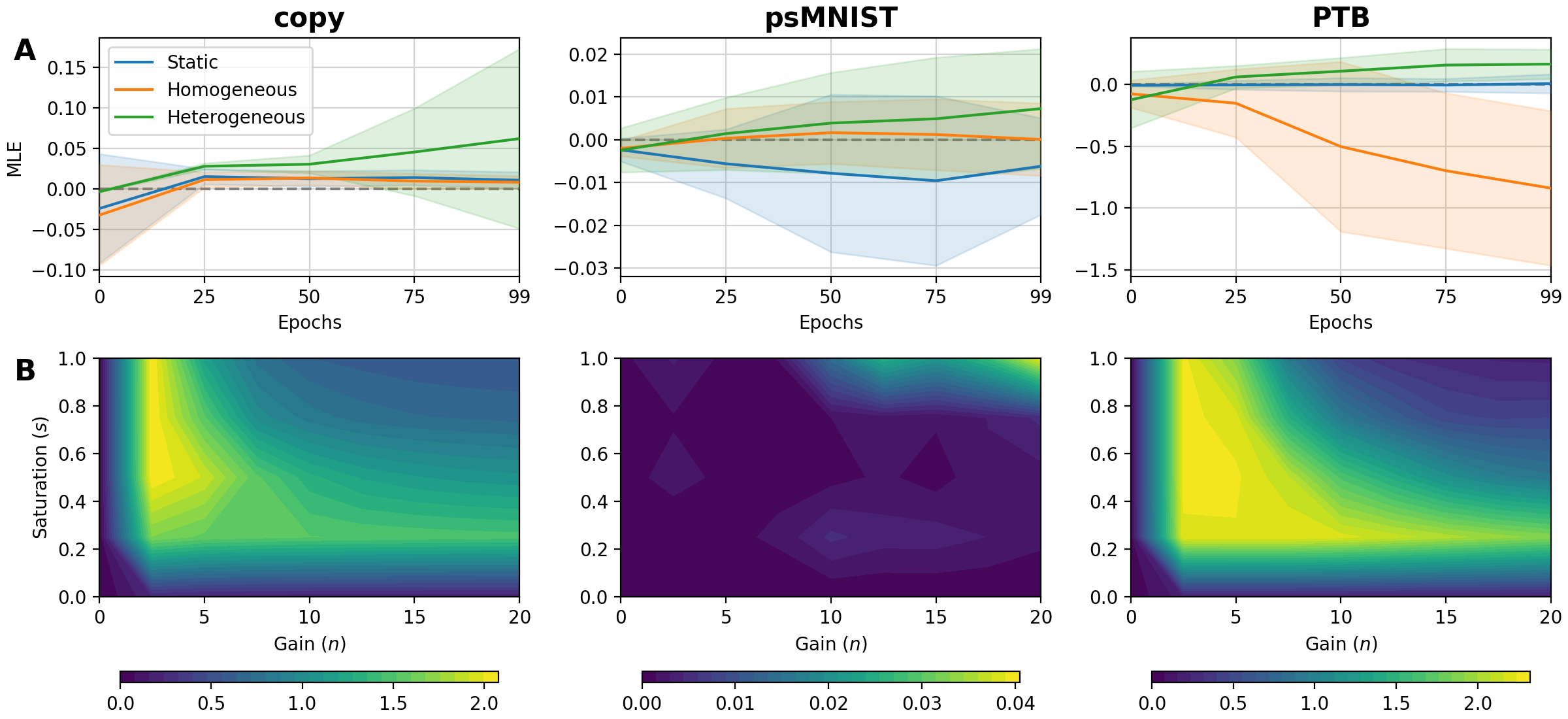}
    \caption{\textit{Signal-propagation measurements}. \textbf{A.} Maximal Lyapunov exponents across epochs. Line corresponds to the average over the trajectories associated with the first quartile in final performance. Shaded regions indicate one standard deviation.  \textbf{B.} Mutual information estimation landscape. Average over initialization (see calculation details in Appendix \S\ref{appendix:MI_calculation}).}
    \label{fig:signals_fig}
\end{figure}

\paragraph{Copy:~hidden-state dynamics and Lyapunov exponents} For its high reliance on computation with hidden-state dynamics, we investigate adaptation in the Copy task through an analysis of the Lyapunov exponents over training. This task requires the networks to retain information, with almost no regard to ongoing computation.  To this end, dynamics at the edge of chaos are particularly well suited as a way through the hidden-states dynamics to store information~\cite{Schoenholz:2016vb,Legenstein:2007th}. 
Lyapunov exponents are kept close to zero in the case of homogeneous adaptation, but become more positive with heterogeneous adaptation. Considering the respective impact of these two adaptation scenarios on performance, this emphasizes that maintaining a system at the edge of chaos is the preferable mechanism. Finally, and unlike the other tasks which will be covered below, mutual information does not provide significant insight on the adaptation dynamics observed.

\paragraph{PTB:~efficient encoding and mutual information} On the other end in terms of computational properties, the PTB character-level prediction task gives more importance to information processing and computation across different timescales. 
Through the different scenarios of adaptation (see Figure \ref{fig:signals_fig}A), the MLEs vary considerably, indicating a wide range of quantitatively different internal dynamics with yet almost the same test performance. For instance, we see a sharp decrease over in the MLEs training for the homogeneous scenario, explained by the associated best performing combinations in Figure \ref{fig:main_performance_fig}B moving towards negative saturation and high gain levels, a region with particularly negative MLEs for random weights (see Figure \ref{fig:Panel_1:shape_MLE_norm}E). 
Hence Lyapunov exponents are not enough to understand adaptation dynamics, however mutual information tells a different story. Figure \ref{fig:signals_fig}B presents the landscape of mutual information $I(H_t, X_{t-1:t})$. The observed shape parameters dynamics in Figure \ref{fig:main_performance_fig}B-right seem evolve to minimize mutual information. This observation suggests a mechanism by which activation functions adapt in tasks with an emphasis on short term computation, and provides an interesting avenue for future exploration.

\paragraph{psMNIST: balancing hidden-state dynamics and encoding efficiency} This task requires an interesting balance between long-term and short-term computational properties. To better understand how these two come into play, we need both the Lyapunov exponents and the landscape of mutual information. 
We observe that the MLEs for the best performing models stay close to zero throughout training for all the adaptation scenarios; in fact, the better the models of an adaptation scenario perform on this task, the closer the MLEs are to zero.
Indeed, by comparing the adaptive dynamics of Figure \ref{fig:main_performance_fig}B-middle and the contour plot of the MLEs of Figure \ref{fig:Panel_1:shape_MLE_norm}E, we find that the adaptation dynamics aim to bring the MLEs closer to zero.
This reinforces our claim above of increased performance in edge-of-chaos regimes for tasks with a reliance on hidden-state dynamics. 
In parallel, we observe that the regions with lowest MI seem to also act as attracting sets for the adaptation dynamics, in a similar way to PTB.
Just as this task requires a balance between hidden-state dynamics and encoding efficiency, both the Lyapunov exponents and mutual information help shed some light the observed adaptation dynamics.

\subsection{Transfer learning experiments with adaptation}\label{ss:results:transfer}

\begin{wrapfigure}{R}{0.4\textwidth}
  \begin{center}
    \includegraphics[width=0.38\textwidth]{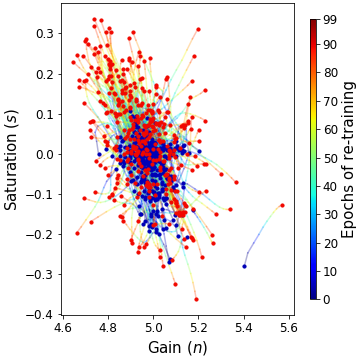}
  \end{center}
  \caption{Trajectories of the shape parameters during retraining on the modified MNIST images.}
  \label{fig:shape_trajectories}
\end{wrapfigure}

In neuroscience, the term adaptation is mostly used to describe processes that occur on short timescales and at a neuron level which have been shown to account for changes in stimulus statistics \cite{adaptation}. This mechanism is naturally linked to the concept of transfer learning in AI where one seeks systems where minimal changes in parameters allow adaptation from learned tasks to novel ones.
To see if changes in single neurons activation could offer transfer advantages in ANNs, we design a novel task using the psMNIST test data set where the images are rotated by 45°. The goal is for a trained network to adapt to this change in input structure by only changing its activation function parameters. To evaluate this, we split rotated images into training and test sets. Each set contains approximately 5k images and the same number of images per digit.
We then briefly retrain heterogeneous activation parameters $(n_i,s_i)$ on this rotated data set using the heterogeneous adaptation scenario. For initialization, we take the parameters (including the $(n_i,s_i)$'s) that resulted from training with normal images, also under the heterogeneous adaptation scenario (see Figure~\ref{fig:main_performance_fig}C-right).
Before retraining, the networks achieved an accuracy of 94\% on the original data set, this fell to 42\% after rotation. Retraining $(n,s)$ allowed the networks to recover classification accuracy up to 56\%. This shows that simply allowing the activation functions to adapt can recover over a quarter of lost performance (over 25\%).

An example of $(n,s)$ trajectories during retraining is showed in Figure \ref{fig:shape_trajectories}. Like in the PTB plot of Figure \ref{fig:main_performance_fig}C, the cloud of $(n,s)$ parameters expands, suggesting that a diversification in nonlinearity shapes is needed to adapt to the change in task and is consistent with neural diversity in the brain. As we further discuss in section \ref{s:conclusion}, this suggests interesting avenues to implement rapid adaptation protocols for ANNs.

\section{Conclusion \& Discussion}\label{s:conclusion}
In this paper, we provide an experimentally-based analysis of the impact of activation functions shape and adaptation in RNNs.
We introduce a novel nonlinear activation function inspired by IF-curves of biological neurons, parameterized by a degree of saturation and gain which permit interpolation between well-known activation functions.
We include these parameters in the network's training procedure, thus allowing single neuron response to adapt over the course of optimization. 
We investigate two scenarios, one where all neurons share the same activation parameters (homogeneous), and another where we enable diversity to emerge with each neuron having distinct activation properties (heterogeneous).
Using numerical experiments on the copy, the permuted sequential MNIST and the character-level language modeling prediction with the Penn Tree Bank corpus tasks, we investigate the impact of adaptation on learning dynamics. 

We first observe that the activation functions achieving the best performance differ from those normally used in machine learning, and depend on the input statistics of the tasks. This suggests that the choice of nonlinearity might not be as inconsequential as currently believed in the ML community. 
This observation is further emphasized when considering the learned activation functions within the adaptation scenarios. We found homogeneous adaptation to improve learning speed and overall performance. While heterogeneous adaptation results are more nuanced we observed that in a task which requires ongoing computations across several timescales (PTB), heterogeneous adaptation gives considerable gains and the solution leads to diversity across neurons.
In our analysis, we highlight the importance of activation functions in signal propagation through an analysis of the evolution of Lyapunov exponents over training, with insights from the mutual information between the last hidden-state and input distributions. We explore the idea that adaptation dynamics may be explained by a trade-off between hidden-states dynamics for long-term dependencies and complex computations and encoding efficiency.
Also, our preliminary results indicate that nonlinearity adaptation can help accuracy after a change in input statistics. 

Finally, we identify a promising use for rapidly adaptive activation functions in ANNs. The results from \S\ref{ss:results:transfer} indicate that individual neuron adaptation can help mitigate changes in input statistics (just like the brain). Single neuron activations are thus a great candidate for an online adaptation process. Here, we envision a single adaptive function that could be learned throughout training, and deployed at inference time, acting locally on individual neurons and having a similar role as input normalization. In contrast, this strategy would be difficult to implement on, say, connectivity matrices. This work is currently in progress and holds promise not only to improve ANNs, but to explore coding properties of neural adaption in the brain by enabling models and experimental predictions.

\section*{Acknowledgments}

We would like to thank Kyle Goyette for useful discussions. This work was partially funded by: IVADO (l'institut de valorisation des donn\'{e}es) [\emph{V.G.}, \emph{G.W.}, \emph{S.H.}], NIH grant R01GM135929 [\emph{G.W.}]; NSERC Discovery Grant (RGPIN-2018-04821), FRQNT Young Investigator Startup Program (2019-NC-253251), and FRQS Research Scholar Award, Junior 1 (LAJGU0401-253188) [\emph{G.L.}] The content is solely the responsibility of the authors and does not necessarily represent the official views of the funding agencies.

\bibliography{main, main_2}

\begin{thebibliography}{1}

\bibitem{Hinton:RMSprop}
G~Hinton, N~Srivastava, and K~Swersky.
\newblock Lecture 6e, rmsprop: Divide the gradient by a running average of its
  recent magnitude.
\newblock 2012.

\bibitem{DBLP:journals/corr/Graves13}
Alex Graves.
\newblock Generating sequences with recurrent neural networks.
\newblock {\em CoRR}, abs/1308.0850, 2013.

\bibitem{Glorot:normal_init}
Xavier Glorot and Yoshua Bengio.
\newblock Understanding the difficulty of training deep feedforward neural
  networks.
\newblock In Yee~Whye Teh and Mike Titterington, editors, {\em Proceedings of
  the Thirteenth International Conference on Artificial Intelligence and
  Statistics}, volume~9 of {\em Proceedings of Machine Learning Research},
  pages 249--256, Chia Laguna Resort, Sardinia, Italy, 13--15 May 2010. PMLR.

\bibitem{Kaiming:init_&_PReLU}
Kaiming He, Xiangyu Zhang, Shaoqing Ren, and Jian Sun.
\newblock Delving deep into rectifiers: Surpassing human-level performance on
  imagenet classification.
\newblock {\em CoRR}, abs/1502.01852, 2015.

\end{thebibliography}


\begin{thebibliography}{10}

\bibitem{Gjorgjieva:2016bt}
Julijana Gjorgjieva, Guillaume Drion, and Eve Marder.
\newblock {ScienceDirect Computational implications of biophysical diversity
  and multiple timescales in neurons and synapses for circuit performance}.
\newblock {\em Current Opinion in Neurobiology}, 37:44--52, April 2016.

\bibitem{adaptation}
Alison~I. Weber, Kamesh Krishnamurthy, and Adrienne~L. Fairhall.
\newblock Coding principles in adaptation.
\newblock {\em Annual Review of Vision Science}, 5(1):427--449, 2019.
\newblock PMID: 31283447.

\bibitem{Vorontsov:2017tt}
Eugene Vorontsov, Chiheb Trabelsi, Samuel Kadoury, and Chris Pal.
\newblock {On orthogonality and learning recurrent networks with long term
  dependencies}.
\newblock {\em arXiv.org}, January 2017.

\bibitem{Hahnloser:2000wl}
RHR Hahnloser, R~Sarpeshkar, M~A Mahowald, R~J Douglas, and S~Seung.
\newblock {Digital selection and analogue amplification coexist in a
  cortex-inspired silicon circuit (vol 405, pg 947, 2000)}.
\newblock {\em Nature}, 408(6815):1012--U24, 2000.

\bibitem{Vecoven:2020ba}
Nicolas Vecoven, Damien Ernst, Antoine Wehenkel, and Guillaume Drion.
\newblock {Introducing neuromodulation in deep neural networks to learn
  adaptive behaviours}.
\newblock {\em PLoS ONE}, 15(1):e0227922--13, January 2020.

\bibitem{Marder:2014du}
Eve Marder, Timothy O{\textquoteright}leary, and Sonal Shruti.
\newblock {Neuromodulation of Circuits with Variable Parameters: Single Neurons
  and Small Circuits Reveal Principles of State-Dependent and Robust
  Neuromodulation}.
\newblock {\em Annual Review of Neuroscience}, 37(1):329--346, July 2014.

\bibitem{gerstner_kistler_naud_paninski_2014}
Wulfram Gerstner, Werner~M Kistler, Richard Naud, and Liam Paninski.
\newblock {\em {Neuronal Dynamics: From Single Neurons to Networks and Models
  of Cognition}}.
\newblock Cambridge University Press, 2014.

\bibitem{SingleNeuronNonlins}
Samuel~P. Muscinelli, Wulfram Gerstner, and Tilo Schwalger.
\newblock How single neuron properties shape chaotic dynamics and signal
  transmission in random neural networks.
\newblock {\em PLOS Computational Biology}, 15(6):1--35, 06 2019.

\bibitem{Crisanti_2018}
A.~Crisanti and H.~Sompolinsky.
\newblock Path integral approach to random neural networks.
\newblock {\em Physical Review E}, 98(6), Dec 2018.

\bibitem{hennequin}
Guillaume Hennequin, Yashar Ahmadian, Daniel~B. Rubin, Máté Lengyel, and
  Kenneth~D. Miller.
\newblock The dynamical regime of sensory cortex: Stable dynamics around a
  single stimulus-tuned attractor account for patterns of noise variability.
\newblock {\em Neuron}, 98(4):846 -- 860.e5, 2018.

\bibitem{PNASnetworks}
N.~Kraynyukova and T.~Tchumatchenko.
\newblock Stabilized supralinear network can give rise to bistable,
  oscillatory, and persistent activity.
\newblock {\em Proceedings of the National Academy of Sciences},
  115(13):3464--3469, 2018.

\bibitem{Fairhall:2001fx}
Adrienne~L Fairhall, Geoffrey~D Lewen, William Bialek, and Robert~R
  de~Ruyter~van Steveninck.
\newblock {Efficiency and ambiguity in an adaptive neural code}.
\newblock {\em Nature Publishing Group}, 412(6849):787--792, August 2001.

\bibitem{Lundstrom:2008km}
Brian~N Lundstrom, Matthew~H Higgs, William~J Spain, and Adrienne~L Fairhall.
\newblock {Fractional differentiation by neocortical pyramidal neurons}.
\newblock {\em Nature Neuroscience}, 11(11):1335--1342, October 2008.

\bibitem{Mease:2013hd}
Rebecca~A Mease, Michael Famulare, Julijana Gjorgjieva, William~J Moody, and
  Adrienne~L Fairhall.
\newblock {Emergence of adaptive computation by single neurons in the
  developing cortex.}
\newblock {\em Journal of Neuroscience}, 33(30):12154--12170, July 2013.

\bibitem{ChaosInNNs}
H.~Sompolinsky, A.~Crisanti, and H.~J. Sommers.
\newblock Chaos in random neural networks.
\newblock {\em Phys. Rev. Lett.}, 61:259--262, Jul 1988.

\bibitem{henaff}
Mikael Henaff, Arthur Szlam, and Yann LeCun.
\newblock Recurrent orthogonal networks and long-memory tasks.
\newblock In Maria~Florina Balcan and Kilian~Q. Weinberger, editors, {\em
  Proceedings of The 33rd International Conference on Machine Learning},
  volume~48 of {\em Proceedings of Machine Learning Research}, pages
  2034--2042, New York, New York, USA, 20--22 Jun 2016. PMLR.

\bibitem{trainingRNNs}
Y.~{Bengio}, T.~{Mikolov}, and R.~{Pascanu}.
\newblock {On the difficulty of training Recurrent Neural Networks}.
\newblock {\em arXiv e-prints}, November 2012.

\bibitem{Bengio:1994do}
Y~Bengio, P~Simard, and P~Frasconi.
\newblock {Learning long-term dependencies with gradient descent is difficult}.
\newblock {\em IEEE Transactions on Neural Networks}, 5(2):157--166, 1994.

\bibitem{Pennington:2018tg}
Jeffrey Pennington, Samuel~S Schoenholz, and Surya Ganguli.
\newblock {The Emergence of Spectral Universality in Deep Networks}.
\newblock {\em arXiv.org}, February 2018.

\bibitem{BenPoole:2016vya}
Ben Poole, Subhaneil Lahiri, Maithra Raghu, Jascha Sohl-Dickstein, and Surya
  Ganguli.
\newblock {Exponential expressivity in deep neural networks through transient
  chaos}.
\newblock {\em arXiv.org}, June 2016.

\bibitem{Tishby:1999;IB}
Naftali Tishby, Fernando~C Pereira, and William Bialek.
\newblock The information bottleneck method.
\newblock {\em ArXiv}, physics/0004057, 1999.

\bibitem{Tishby:2017}
Ravid Shwartz{-}Ziv and Naftali Tishby.
\newblock Opening the black box of deep neural networks via information.
\newblock {\em CoRR}, abs/1703.00810, 2017.

\bibitem{nnRNN}
G.~{Kerg}, K.~{Goyette}, M.~{Puelma Touzel}, G.~{Gidel}, E.~{Vorontsov},
  Y.~{Bengio}, and G.~{Lajoie}.
\newblock {Non-normal Recurrent Neural Network (nnRNN): learning long time
  dependencies while improving expressivity with transient dynamics}.
\newblock {\em arXiv e-prints}, May 2019.

\bibitem{LSTM}
Sepp Hochreiter and J\"{u}rgen Schmidhuber.
\newblock Long short-term memory.
\newblock {\em Neural Comput.}, 9(8):1735–1780, November 1997.

\bibitem{PTB}
Mitchell~P. Marcus, Mary~Ann Marcinkiewicz, and Beatrice Santorini.
\newblock Building a large annotated corpus of english: The penn treebank.
\newblock {\em Comput. Linguist.}, 19(2):313–330, June 1993.

\bibitem{PyTorch}
Adam Paszke, Sam Gross, Soumith Chintala, Gregory Chanan, Edward Yang, Zachary
  DeVito, Zeming Lin, Alban Desmaison, Luca Antiga, and Adam Lerer.
\newblock Automatic differentiation in pytorch, 2017.

\bibitem{Adam}
Diederik~P. Kingma and Jimmy Ba.
\newblock Adam: {A} method for stochastic optimization.
\newblock In Yoshua Bengio and Yann LeCun, editors, {\em 3rd International
  Conference on Learning Representations, {ICLR} 2015, San Diego, CA, USA, May
  7-9, 2015, Conference Track Proceedings}, 2015.

\bibitem{ReLU_Glorot_Bengio}
Xavier Glorot, Antoine Bordes, and Yoshua Bengio.
\newblock Deep sparse rectifier neural networks.
\newblock In Geoffrey Gordon, David Dunson, and Miroslav Dudík, editors, {\em
  Proceedings of the Fourteenth International Conference on Artificial
  Intelligence and Statistics}, volume~15 of {\em Proceedings of Machine
  Learning Research}, pages 315--323, Fort Lauderdale, FL, USA, 11--13 Apr
  2011. PMLR.

\bibitem{Schoenholz:2016vb}
Samuel~S Schoenholz, Justin Gilmer, Surya Ganguli, and Jascha Sohl-Dickstein.
\newblock {Deep Information Propagation}.
\newblock {\em arXiv.org}, November 2016.

\bibitem{Legenstein:2007th}
R~Legenstein and W~Maass.
\newblock {Edge of chaos and prediction of computational performance for neural
  circuit models}.
\newblock {\em Neural Networks}, 20(3):323--334, 2007.

\bibitem{bennettin}
Giancarlo Benettin, L~Galgani, Antonio Giorgilli, and Marie Strelcyn.
\newblock Lyapunov characteristic exponents for smooth dynamical systems and
  for hamiltonian systems; a method for computing all of them. part 1: theory.
\newblock {\em Meccanica}, 15:9--20, 03 1980.

\bibitem{DIECI1995}
Luca Dieci and Erik S.~Van Vleck.
\newblock Computation of a few lyapunov exponents for continuous and discrete
  dynamical systems.
\newblock {\em Applied Numerical Mathematics}, 17(3):275 -- 291, 1995.
\newblock Special Issue on Numerical Methods for Ordinary Differential
  Equations.

\bibitem{arnold}
Ludwig Arnold.
\newblock {\em Random Dynamical Systems}.
\newblock Springer, 1998.

\bibitem{Cover2006}
Thomas~M. Cover and Joy~A. Thomas.
\newblock {\em Elements of Information Theory 2nd Edition (Wiley Series in
  Telecommunications and Signal Processing)}.
\newblock Wiley-Interscience, July 2006.

\bibitem{DBLP:conf/nips/GaoKOV17}
Weihao Gao, Sreeram Kannan, Sewoong Oh, and Pramod Viswanath.
\newblock Estimating mutual information for discrete-continuous mixtures.
\newblock In Isabelle Guyon, Ulrike von Luxburg, Samy Bengio, Hanna~M. Wallach,
  Rob Fergus, S.~V.~N. Vishwanathan, and Roman Garnett, editors, {\em Advances
  in Neural Information Processing Systems 30: Annual Conference on Neural
  Information Processing Systems 2017, 4-9 December 2017, Long Beach, CA,
  {USA}}, pages 5986--5997, 2017.

\end{thebibliography}
\bibliographystyle{unsrt}

\newpage
\appendix
\begin{center}
{\bf Supplementary Material for:}\\
Advantages of biologically-inspired adaptive neural activation in RNNs during learning
\end{center}

\setcounter{thm}{0}
\setcounter{lemma}{0}
\setcounter{prop}{0}

\section{Experimental details}\label{appendix:task_setup}

\subsection{Experimental setups}

\begin{table}[H]
    \centering
    \begin{tabular}{l| c c c c c}
    \toprule
        Task & LR & hid & LR-scheduler & Rec. init. & In init.\\
    \midrule
        copy &  $10^{-4}$ & 128 & & Henaff & Glorot normal \\
        psMNIST & $10^{-4}$ & 400 & & Henaff & Kaiming \\
        PTB & $10^{-4}$ & 600 & ReduceLROnPlateau & Henaff & Kaiming\\
    \bottomrule
    \end{tabular}
    \vspace{8pt}
    \caption{\textit{Task-dependent hyperparameters}, where "hid" is hidden state size, "LR" is learning rate, "Rec. init." and "In init." are the initialization scheme for respectively the state transition matrix and the input weights.}
    \label{tab:hyperparams}
\end{table}

We trained the networks for 100K iterations for the Copy task, and 100 epochs for the others. We investigated different learning rates (LR $\in \{10^{-3}, 10^{-4}, 10^{-5}, 10^{-6}\}$), and settled on those in the Table \ref{tab:hyperparams} for each task.

Independently of the task, we used Cross-entropy loss as our loss function and the Adam \cite{Adam} optimizer. We experimented with the RMSprop optimizer (introduced in \citeNew{Hinton:RMSprop}, first used in \citeNew{DBLP:journals/corr/Graves13}) with smoothing constant $\alpha = 0.99$ and no weight decay, which yielded similar results. The initialization schemes in Table \ref{tab:hyperparams} for the recurrent weights and input weights refer to :
\begin{itemize}
    \item Henaff : Random orthogonal matrices following \cite{henaff}.
    \item Glorot Normal : Normally distributed weights following \citeNew{Glorot:normal_init}. Implemented in the \texttt{torch.nn} initialization library as \verb+xavier_normal_+.
    \item Kaiming : Normally distributed weights following \citeNew{Kaiming:init_&_PReLU}. Also known as He initialization. Implemented in the \texttt{torch.nn} initialization library as \verb+kaiming_normal_+.
\end{itemize}

\paragraph{Shape parameters} The initialization grid for the shape parameters used throughout this work is $N\times S$, where $N= \{1.0\} \cup \{1.25k\ :\ 1 \leq k \leq 16\}$ and $S = \{0.0,\ 0.25,\ 0.5,\ 0.75,\ 1.0\}$ such that $|N| = 17$ and $|S| = 5$. In the heterogeneous adaptation scenario, both $\boldsymbol{n}$ and $\boldsymbol{s}$ vectors are initialized with the same value for each component.

\subsection{Pytorch autograd implementation of gamma}

We implement $\gamma(x;n,s)$ as a Pytorch autograd Function with corresponding Pytorch Module. To allow for activation function adaptation, we further include the shape parameters in the backpropagation algorithm. We do so by defining the gradient of $\gamma$ with respect to the input and parameters. We can rewrite $\gamma(x;n,s)$ as :
\begin{equation}
    \gamma(x;n,s) = \frac{(1-s)}{n}\gamma_1(nx) + s\sigma(nx)
\end{equation}
where $\sigma(x)$ is the sigmoid activation function. With this notation, the partial derivatives of $\gamma$ with respect to $x$ (or total derivative), $n$ and $s$ are
\begin{align}
    \gamma'(x;n,s) = \frac{\partial}{\partial x}\gamma(x;n,s) &=  (1-s)\sigma(nx) + n s \sigma(nx)(1- \sigma(nx))\\ 
    \frac{\partial}{\partial n} \gamma(x;n,s) &= \frac{1-s}{n}(x \sigma(nx) - \frac{\gamma_1(nx)}{n}) + s x \sigma (nx)(1-\sigma(nx))  \label{eq:gamma'_n} \\
    \frac{\partial}{\partial s} \gamma(x;n,s) &= \sigma(nx)-\frac{\gamma_1(nx)}{n}
\end{align}

\section{A primer on Lyapunov exponents} \label{LE_intro}

In this section we are first going to give a bit of theoretical background on Lyapunov exponents. Exponential explosion and vanishing of long products of Jacobian matrices is a long studied topic in dynamical systems theory, where an extensive amount of tools have been developed in order to understand these products. Thus one can hope to leverage these tools in order to better understand the exploding and vanishing gradient problem in the context of RNNs. 
 
\subsection{Definition of Lyapunov exponents}
Let $F:X\rightarrow X$ be a continuously differentiable function, and consider the discrete dynamical system $(F,X,T)$ defined by
\begin{align}
    x_{t+1} &= F(x_t)
\end{align}
for all $t\in T$, where $X$ is the phase space, and $T$ the time range. We would like to gain an intuition for how trajectories of the mentioned dynamical system behave under small perturbations. 

Let $x_t$ and $x_t'$ be two trajectories with initial conditions $x_0$ and $x_0'$, such that $|x_0-x_0'|$ is sufficiently small. 

Defining $\epsilon_t = x_t'-x_t$, we get by the first order Taylor expansion 
\begin{align}
    x_{t+1}' &= F(x_t')\\
    &=F(x_t+\epsilon_t)\\
    &=F(x_t)+DF(x_t)\cdot\epsilon_t+O(|\epsilon_t|^2)\\
    &=x_{t+1}+DF(x_t)\cdot\epsilon_t+O(|\epsilon_t|^2)
\end{align}

Substracting $x_{t+1}$ both sides we get the variational equation
\begin{align}
    \epsilon_{t+1} &= DF(x_t) \cdot \epsilon_t + O(|\epsilon_t|^2)\\
    &\approx \prod_{k=0}^{t}DF(x_k)\cdot \epsilon_0\\
    &=DF^{t+1}(x_0) \cdot \epsilon_0
\end{align}
(Here $DF^{t+1}(x_0)$ is an abuse of notation for the Jacobian of the $(t+1)$-th iterate of $F$, evaluated at $x_0$). Intuitively the ratio $\frac{\|\epsilon_t\|}{\|\epsilon_0\|} = \frac{\|DF^{t}(x_0)\cdot \epsilon_0\|}{\|\epsilon_0\|}$ describes the expansion/contraction rate after $t$ time steps if our initial perturbation was $\epsilon_0$, which motivates the following definition:

Let $x_0,w\in X$, define \begin{align}
    \lambda(x_0,w) &\overset{\mathrm{def}}{=} \lim_{m \rightarrow \infty}\frac{1}{m} \ln \prod_{t=1}^{m} \frac{\|DF^{t}(x_0)\cdot w\|}{\|w\|}\\ 
    & = \lim_{m \rightarrow \infty}\frac{1}{m}  \sum_{t=1}^{m} \ln \frac{\|DF^{t}(x_0)\cdot w\|}{\|w\|} \end{align}
Thus $\lambda(x_0,w)$ measures the average rate of expansion/contraction over an infinite time horizon of the trajectory starting at $x_0$, if it has been given an initial perturbation $w$. Note that once $x_0$ and $w$ have been fixed, the quantity $\lambda(x_0,w)$ is intrinsic to the discrete dynamical system defined by $x_{t+1} = F(x_t)$. We call $\lambda(x_0,w)$ a \textbf{Lyapunov exponent} of the mentioned dynamical system.

Since the Lyapunov exponents describe the the average rate of expansion/contraction for long products of Jacobian matrices, it doesn't sound too surprising that they might provide an interesting perspective to study the exploding and vanishing gradient problem in RNNs. To give a complete picture of the analogy to RNNs, one can think of $x_t$ as the hidden state at time $t$, and $F$ can be seen as the function defined in the RNN cell. The only difference is that in RNNs we have inputs at every time steps, and thus the function $F$ changes at every time step. This is the distinction between autonomous and non-autonomous dynamical systems, which is explained in more detail in the upcoming subsection \ref{link_to_rnn}.

Finally, let us remark that the expression in the above definition of Lyapunov exponents is not always well defined. This will be the topic of the next subsection \ref{oseledets_section}, where we are presenting Oseledets theorem which gives exact conditions for when the above expression in well-defined.  

\subsection{Oseledets theorem} \label{oseledets_section}

As already stated, we bypassed the fact that the limit in the definition of $\lambda(x_0,w)$ might not actually exists. In fact this is the result of the well-known \textit{Oseledets theorem}, but before stating the theorem let us point out a definition.

\paragraph{Definition.} A \textit{cocycle} of an autonomous dynamical system $(F,X,T)$ is a map $C: X \times T \rightarrow \mathbb{R}^{n\times n}$ satisfying:
\begin{itemize}
    \item $C(x_0,0) = \textrm{Id}$ 
    \item $C(x_0,t+s) = C(x_t,s)C(x_0,t)$ for all $x_0 \in X$ and $s,t \in T$
\end{itemize}

\paragraph{Oseledets theorem.} (sometimes referred to as Oseledets \textit{multiplicative ergodic theorem}) Let $\mu$ be an ergodic invariant measure on $X$, and let $C$ be a cocycle of a dynamical system $(F,X,T)$ such that for each $t\in T$, the maps $x \mapsto \log \|C(x,t)\|$ and $x \mapsto \log \|C(x,t)^{-1}\|$ are $L^1$-integrable with respect to $\mu$. Then for $\mu$-almost all $x$ and each non-zero vector $w \in \mathbb{R}^n$ the limit
\begin{align}
    \lambda(x,w)&=\lim _{t \rightarrow \infty} \frac{1}{t} \ln \frac{\|C(x, t) w\|}{\|w\|}\end{align}
exists and assumes, depending on $w$ but not on $x$, up to $n$ different values, called the Lyapunov exponents (giving rise to a more general definition) 

One can prove that the following matrix limit
\begin{align}
\Lambda &= \lim _{t \rightarrow \infty}[C(x,t)^TC(x,t)]^{1/2t}
\end{align}
exists, is symmetric positive-definite and its log-eigenvalues are the Lyapunov exponents. We call $\Lambda$ the \textit{Oseledets matrix}.

In order to make this definition a little bit more intuitive, let us come back to our original situation, and note that the terms $\prod_{k=0}^{t}DF(x_k)=DF^{t+1}(x_0)$ define a cocycle verifying the conditions of the theorem. Thus, in this case, the Lyapunov exponents are not only well defined, but there are up to $n$ distinct ones of them, and they are the log-eigenvalues of the following Oseledets matrix: 

\begin{align}\Lambda &= \lim _{t \rightarrow \infty}[DF^{t}(x_0)^T\cdot DF^{t}(x_0)]^{1/2t}\end{align}

Let us now consider the singular value decomposition of $DF^{t}(x_0)$,
\begin{align} DF^{t}(x_0)V(x_0,t)&=U(x_0,t)\Sigma(x_0,t)\end{align}
where $\Sigma(x_0,t)$ is a diagonal matrix composed of the singular values $\sigma_1(x_0,t)\geq\ldots\geq\sigma_n(x_0,t) \geq 0$, and $U(x_0,t)$ as well as $V(x_0,t)$ are orthogonal matrices, composed column-wise of the left and right singular vectors respectively. Then 

\begin{align}\Lambda &= \lim _{t \rightarrow \infty} V(x_0,t)^T \Sigma(x_0,t)^{1/t}V(x_0,t)\end{align}

Thus, for large $t$, the log-eigenvalues of $\Lambda$ can be approximated by $\frac{1}{t}\ln \sigma_i(x_o,t)$'s, which can be thought of as the average singular value along an infinite time horizon. \ul{It turns out that for ergodic systems, the Lyapunov exponents are independent of initial conditions $x_0$. Thus, intuitively, Lyapunov exponents are topological quantities intrinsic to the dynamical system that describe the average amount of instability along infinite time horizons.}

In order to understand how this instability manifests along each direction, let us further look what we can say about the vectors associated with the individual Lyapunov exponents. If we denote $\lambda^{(1)}\geq \lambda^{(2)}\geq \ldots \geq \lambda^{(s)}$ the \textit{distinct} Lyapunov exponents, and $v_i(x_0)$ the corresponding vector of the matrix $\lim _{t \rightarrow \infty} V(x_0,t)$, then let us define the nested subspaces
\begin{align}S_j(x_0)= \textrm{span}\{v_i(x_0) | i=j,j+1,\ldots,s\}\end{align} for all $j=1,2,\ldots,s$, and take a vector $w_j(x_0) \in S_j(x_0) \setminus S_{j+1}(x_0)$. Then $w_j(x_0)$ is orthogonal to all $v_i(x_0)$ with $i< j$, and has a non-zero projection onto $v_j(x_0)$ since $v_j(x_0) \notin S_{j+1}(x_0)$, and thus 

\begin{align} \|DF^t(x_0) \cdot w_j(x_0)\| &\sim e^{\lambda^{(j)}t}\end{align}

In particular, since $S_1(x_0)$ is the whole phase space $X$, and $S_2(x_0)$ is only a hyperplane in $X$ (a subset of Lebesgue measure zero), we have that for "almost all" $w \in X$:

\begin{align} \|DF^t(x_0)\cdot w\| &\sim e^{\lambda^{(1)}t} \end{align}

hence aligning with the direction of maximum Lyapunov exponent (MLE). In other words a randomly chosen vector, has a non-zero projection in the direction of the MLE with probability $1$, and thus over time the effect of the other exponents will become negligible. This motivates taking the MLE as a way of measuring the overall amount of stability or instability of a dynamical system. One typically distinguishes the cases, where the MLE is negative, zero and positive.

Thus computing MLEs, LEs and their corresponding subspaces can be a useful tool to understand the average expansion/ contraction rate as well as the corresponding directions of gradients in recurrent neural networks. 

\subsection{The QR algorithm}
It is generally not advised to calculate the Lyapunov exponents and the associated vectors using $DF^t(x_0)$ as this matrix becomes increasingly ill-conditioned. There is a known algorithm that in most cases allows to provide good estimates, called the \textit{QR algorithm}.

As a preliminary remark, let us emphasize that the right singular vectors of $DF(x_{t+1})$ do not necessarily match the left singular vectors of $DF(x_t)$, thus simply applying the singular value decomposition in order to calculate the Lyapunov exponents does not work.

Let us denote $J_t = DF(x_t)$ for each time step $t=0,1,2,\ldots$, then lets us pick an orthogonal matrix $Q_0$, and compute $Z_0=J_0 Q_0$. Then let us perform the QR decomposition $Z_0=Q_1R_1$. Let us further assume that $J_0$ is invertible and we are imposing that the diagonal elements of $R_1$ are non-negative (which we can), thus making the QR decomposition unique. 

In the next step, we compute $Z_1=J_1Q_1$ and perform the QR decomposition $Z_1=Q_2R_2$, where again we are imposing the diagonal elements of $R_2$ to be non-negative.

Continuing in this fashion at each time step $k$, we then have the identity  $J_k = Q_{k+1}R_{k+1}Q_k^T$, and thus 
\begin{align}DF^{t+1}(x_0) &= \prod_{k=0}^{t}J_k\\ 
&= Q_{t+1}(R_t\cdot \ldots \cdot R_1)Q_0^T\end{align}

It turns out that, as long as the dynamical system is "regular", we can then compute the $i$-th Lyapunov exponent via 
\begin{align}\lambda_i &= \lim_{t \rightarrow \infty} \frac{1}{t} \sum_{k=1}^{t} \ln(R_k)_{ii}\end{align}
where the Lyapunov exponents are ordered $\lambda_1\geq \lambda_2\geq \ldots \geq \lambda_n$ as explained in \cite{bennettin} and \cite{DIECI1995}.

\subsection{Link to RNNs} \label{link_to_rnn}

Recalling the update equation of an RNN:

\begin{align}h_{t+1} &= \phi(Vh_t + Ux_{t+1}+b)\end{align}

for $t=0,1,\ldots$, and by denoting $F(h,x) = \phi(Vh+Ux+b)$, we can see that 
\begin{align}\tilde h_{t+1} &= F(\tilde h_t,0)\end{align}
defines an autonomous discrete dynamical system (DS1), while

\begin{align}h_{t+1} &= F(h_t,x_{t+1})\end{align}

defines a non-autonomous discrete dynamical system (DS2).

For (DS1), the machinery that we have developed over the last subsections is directly applicable, as we are in the autonomous case. For instance, we can compute the Lyapunov exponents of recurrent neural network over the course of training using the QR algorithm, and in particular observe the evolution of the maximum Lyapunov exponent (MLE), as a means to measure the amount of instability or chaos in the network. For example in the case of a linear RNN with a unitary or orthogonal connectivity matrix, all LEs are equal to zero, and thus no expansion nor contraction is happening. If all LEs are negative, we are in the contracting regime, where every point eventually will approach an attractor, thus producing a vanishing gradient. For instance, \cite{Bengio:1994do} showed that storing information in a fixed-size state vector (as is the case in a vanilla RNN) over sufficiently long time horizon in a stable way necessarily leads to vanishing gradients when back-propagating through time (here stable means insensitive to small input perturbations).

The natural question arises whether and to what extent the machinery will stay valid for (DS2). It turns out that one can use the theory of Random Dynamical Systems Theory, where Oseledet's multiplicative ergodic theorem holds under some stationarity assumption of the underlying distribution generating the inputs $x_t$ as stated in \cite{arnold}. However in this paper we are just we are just making use of the machinery developed for (DS1), by computing Lyapunov exponents for trained RNNs but computed without inputs ($x_t =0$ for all $t$).

\section{A primer on mutual information} \label{MI_intro}

In this section we will give some basic definitions and propositions in regards to mutual information loosely following the book \cite{Cover2006}.

\subsection{Entropy}
\begin{defn}[Entropy]\label{def:entropy} Let $X$ be a discrete random variable with support $\mathcal{X}$ and probability density function $p_X$, then we define the \textit{entropy} of $X$ as 
\begin{align}
    \mathcal{H}(X) &\myeq -\mathbb{E}_{x \sim p_X} [\log (p_X(x))]\\
    &= - \sum_{x \in \mathcal{X}} p_X(x) \log (p_X(x))
\end{align}
\end{defn}

\paragraph{Intuition.} In Shannon's "fundamental problem of communication" a receiver needs to identify what data was generated by data generating source, based on the signal it receives via a potentially noisy communication channel. Intuitively, the event of the data generating source producing a low-probability value carries a lot more "information" than if it were a high-probability value. One way to quantify this "information" is via the quantity $-\log p$, where $p$ is the probability of the value generated by the source. As we can see for $p=1$, the amount of information becomes $-\log p = 0$, which is the minimal possible value, while $-\log p$ approaches infinity, as $p \rightarrow 0^{+}$. Then we can see that the entropy $\mathcal{H}(X)$ is the expected amount of "information" carried by $X$ (here $X$ can be seen as the random variable associated to the underlying data generating distribution of the source).

\begin{defn}[Conditional Entropy]\label{def:conditional_entropy}
If we further consider a discrete random variable $Y$ with support $\mathcal{Y}$ and probability density function $p_Y$, we define the \textit{conditional entropy} of a discrete random variable $Y$ given $X$ as
\begin{align}
    \mathcal{H}(X|Y) &\myeq - \mathbb{E}_{(x,y) \sim p_{X,Y}} [\log (p_{Y|X}(y|x))]\\
    &= -\sum_{x \in \mathcal{X}} \sum_{y \in \mathcal{Y}} p_{X,Y}(x,y) \log (p_{Y|X}(y|x))
\end{align}
where $p_{X,Y}$ is the joint probability density function of $(X,Y)$ and $p_{Y|X}$ is the probability density function of the conditional distribution $Y|X$.
\end{defn}

\paragraph{Intuition.} If we use Definition \ref{def:entropy} for the random variable $X|Y=y$, we get: 
\begin{align}
    \mathcal{H}(X|Y=y) = - \sum_{x\in \mathcal{X}} \mathbb{P}(X=x|Y=y) \log \mathbb{P}(X=x|Y=y)
\end{align} which measures the average amount of "information" carried by $X$ given the event $Y=y$. One can then verify that $\mathcal{H}(X|Y)$ is the result of averaging $\mathcal{H}(X|Y=y)$ over all possible $y\in \mathcal{Y}$:

\begin{align}
    \mathbb{E}_{y \sim p_Y}\left[\mathcal{H}(X|Y=y)\right] &= \sum_{y\in \mathcal{Y}} \mathbb{P}(Y=y) \cdot \mathcal{H}(X|Y=y)\\
    &= -\sum_{x\in \mathcal{X},y\in \mathcal{Y}} \mathbb{P}(Y=y) \cdot \mathbb{P}(X=x|Y=y) \cdot  \log \mathbb{P}(X=x|Y=y)\\
    &= -\sum_{x\in \mathcal{X},y\in \mathcal{Y}} \mathbb{P}(X=x,Y=y) \cdot \log \mathbb{P}(X=x|Y=y)\\
    &= \mathcal{H}(X|Y)
\end{align}
Hence we stay consistent with the intuition from Definition \ref{def:entropy}: $\mathcal{H}(X|Y)$ describes the average amount of "information" carried by $X$ given $Y$. In the extreme case, $\mathcal{H}(X|Y) = 0$, if $X$ is fully determined by the knowledge of $Y$.
\begin{defn}[Joint Entropy] \label{def:joint_entropy}
Finally we define the \textit{joint entropy} of $X$ and $Y$ as 
\begin{align}
    \mathcal{H}(X,Y) &\myeq - \mathbb{E}_{x,y \sim p_{X,Y}} [\log (p_{X,Y}(x,y))]\\
    &= -\sum_{x \in \mathcal{X}, y \in \mathcal{Y}} p_{X,Y}(x,y) \log (p_{X,Y}(x,y))
\end{align}
\end{defn}

\paragraph{Intuition.} In order to see how Definition \ref{def:joint_entropy} relates to Definitions \ref{def:entropy} and \ref{def:conditional_entropy}, let us observe that $\mathcal{H}(X,Y) = \mathcal{H}(Y) + \mathcal{H}(X|Y)$. In fact, let us imagine we would first observe $Y$ and then observe $X$. $\mathcal{H}(Y)$ describes the average amount of information carried by $Y$, and $\mathcal{H}(X|Y)$ is the average amount of information carried by $X$ once $Y$ is known. It thus doesn't sound too surprising that this equates exactly the amount of information carried by both $X$ and $Y$. Let us verify this mathematically, 

\begin{align}
\mathcal{H}(X,Y) &= -\sum_{x\in \mathcal{X},y \in \mathcal{Y}} \mathbb{P}(X=x, Y=Y) \log \mathbb{P}(X=x,Y=y) \\
&= -\sum_{x\in \mathcal{X},y \in \mathcal{Y}} \mathbb{P}(X=x, Y=Y) \cdot \left[ \log \mathbb{P}(X=x|Y=y)+ \log \mathbb{P}(Y=y)\right]\\
&= \breaktowidth{8.3cm}{-\sum_{x\in \mathcal{X},y \in \mathcal{Y}} \mathbb{P}(X=x, Y=Y) \cdot \log \mathbb{P}(X=x|Y=y) -\sum_{x\in \mathcal{X},y \in \mathcal{Y}} \mathbb{P}(X=x, Y=Y) \cdot \log \mathbb{P}(Y=y)} \\
&= \mathcal{H}(X|Y) - \sum_{y \in \mathcal{Y}} \mathbb{P}( Y=Y) \cdot \log \mathbb{P}(Y=y)\\
&= \mathcal{H}(X|Y) + \mathcal{H}(Y)
\end{align}
The identity $\mathcal{H}(X,Y) = \mathcal{H}(X|Y) + \mathcal{H}(Y)$ is also referred to as the \textit{chain rule} of conditional entropy, which can be generalized to multiple random variables $X_1, \ldots, X_n$ (via a straightforward proof by induction) as
\begin{align}
    \mathcal{H}(X_1,\ldots X_n) &= \sum_{i=1}^n\mathcal{H}(X_i|X_1,\ldots,X_{i-1})
\end{align}
\subsection{Kullback-Leibler Divergence}
\begin{defn}[Kullback-Leibler Divergence]
The \textit{KL-divergence} between two discrete probability distributions $p$ and $q$ with common support $\mathcal{X}$ is defined by 
\begin{align}
    D_{KL}(p||q) &\myeq \mathbb{E}_{x \sim p}\left[\log \left(\frac{p(x)}{q(x)}\right)\right]\\
    &= \sum_{x \in \mathcal{X}} p(x) \log\left(\frac{p(x)}{q(x)}\right) \\
\end{align}
\end{defn}

\paragraph{Intuition.} The KL-divergence between $p$ and $q$, or also called the \textit{relative entropy} of $p$ with respect to $q$, measures the information lost when $q$ is used to approximate $p$. In other words, it quantifies how much $q$ deviates from $p$. Note that it does not satisfy the mathematical properties of a definition of a metric. We can rewrite 
\begin{align}
    D_{KL}(p||q) &= \textrm{CE}(p,q) - \mathcal{H}(p)
\end{align}
where \begin{align}\textrm{CE}(p,q) = - \sum_{x \in \mathcal{X}} p(x) \log q(x) \end{align} also called the \textit{cross-entropy} between $p$ and $q$, while \begin{align}\mathcal{H}(p)= \textrm{CE}(p,p) = - \sum_{x \in \mathcal{X}} p(x) \log p(x)\end{align} Hence if $p=q$ then $D_{KL}(p||q) = 0$. 
\subsection{Mutual information}
\begin{defn}[Mutual Information] Let $X$ and $Y$ be two random variables with probability density functions $p_X$ and $p_Y$ respectively and joint distribution $p_{X,Y}$, then we define the \textit{mutual information} of $X$ and $Y$ by 
\begin{align}
    \mathcal{I}(X;Y) &\myeq D_{KL}(p_{X,Y} || p_X \otimes p_Y)\\
    &= \sum_{x\in \mathcal{X}, y \in \mathcal{Y}} p_{X,Y}(x,y) \log \frac{p_{X,Y}(x,y)}{p_X(x) \cdot p_Y(y)}
\end{align}    
\end{defn}

\paragraph{Intuition.} Going back to the intuition of the KL-divergence, we see that the mutual information $\mathcal{I}(X;Y)$ between $X$ and $Y$, measures the information lost when approximating $(X,Y)$ as a pair of independent random variables, when in reality they are not. Alternatively, one can interpret $\mathcal{I}(X;Y)$ to measure the amount of "information" $X$ and $Y$ share. Note that when $X$ and $Y$ are indeed independent, $X$ and $Y$ do not share any information and  $\mathcal{I}(X;Y) =0$. Meanwhile, when $X=Y$, we can see that $\mathcal{I}(X;Y) = \mathcal{H}(X)$. We can thus see that mutual information and entropy are intimately related, which motivates an equivalent definition of the mutual information $\mathcal{I}(X;Y)$ in terms of entropy and conditional entropy. In fact one can prove that 

\begin{align}
    \mathcal{I}(X;Y) 
    &= \mathbb{E}_Y \left[D_{KL}(p_{X|Y}||p_X)\right]\\
    &= \mathcal{H}(X) - \mathcal{H}(X|Y)\\
    &= \mathcal{H}(Y) - \mathcal{H}(Y|X)\\
    &= \mathcal{H}(X) + \mathcal{H}(Y)- \mathcal{H}(X,Y)\\
    &= \mathcal{H}(X,Y) - \mathcal{H}(X|Y) -\mathcal{H}(Y|X)
\end{align}
where in the last two equalities we have used the chain rule of conditional entropy. Here $\mathcal{H}(X) - \mathcal{H}(X|Y)$ can be seen as the information gain carried by $X$, did we not know anything about $Y$. Also we see that the  mutual information $\mathcal{I}(X;Y)$ is symmetric in $X$ and $Y$. 

\begin{defn}[Conditional Mutual Information]
Further, let $Z$ be a third random variable, then we define the conditional mutual information between $X$ and $Y$ given $Z$ via 
\begin{align}
    \mathcal{I}(X;Y|Z) &\myeq \mathbb{E}_Z\left[D_{KL}(p_{X,Y|Z} || p_{X|Z} \otimes p_{Y|Z})\right]\\
    &= \mathcal{H}(X|Z) - \mathcal{H}(X|Y,Z)\\
    &= \mathcal{H}(Y|Z) - \mathcal{H}(Y|X,Z)\\
    &= \mathbb{E}_Z\left[\mathbb{E}_{Y|Z} \left[D_{KL}(p_{X|Y,Z}||p_{X|Z})\right]\right] \label{cond_IM}
\end{align}
where in the last equality we made use of the identity $p_{X|Y,Z}(x|y,z) =  \frac{p_{X,Y|Z}(x,y|z)}{p_{Y|Z}(y|z)}$ for all $x$, and all $(y,z) \in \textrm{supp}(p_Y) \times \textrm{supp}(p_Z)$. From this follows the following the chain rule for mutual information.
\end{defn}
\begin{lemma}[Chain Rule for Mutual Information]
The multivariate mutual information between the the random variables $X$, $Y$ and $Z$ is defined by 
\begin{align}
    \mathcal{I}(X;Y,Z) &= \mathcal{I}(X;Z) + \mathcal{I}(X;Y|Z)
\end{align}
\end{lemma}
\begin{proof}
\begin{align}
    \mathcal{I}(X,Y;Z) &= \mathcal{H}(X,Y) - \mathcal{H}(X,Y|Z)\\
    &= \mathcal{H}(X)+ \mathcal{H}(Y|X) - \mathcal{H}(X|Z) - \mathcal{H}(Y|X,Z)\\
    &= (\mathcal{H}(X) - \mathcal{H}(X|Z)) + (\mathcal{H}(Y|X)- \mathcal{H}(Y|X,Z))\\
    &=  \mathcal{I}(X;Z) +  \mathcal{I}(Y;Z|X)
\end{align}
\end{proof}
By a straightforward induction argument we can prove the following more general version of the chain rule.

\begin{prop}[Chain Rule for Mutual information]\label{CRMI}
Let $X$ and $Z_1,\ldots,Z_n$ be random variables, then \begin{align}
    \mathcal{I}(X;Z_1,\ldots,Z_n) = \sum_{i=1}^n \mathcal{I}(X;Z_i|Z_1,\ldots,Z_{i-1})
\end{align}
\end{prop}
\subsection{Information Bottleneck and application to RNNs}\label{appendix:ss:IB}
The Information Bottleneck (IB) principle~\cite{Tishby:1999;IB,Tishby:2017} is an information theoretic method for extracting relevant information from an input variable $X$ about an output variable $Y$, via a representation $H$ of $X$ with respect to $Y$. In fact, one would like to minimize the Lagrangian 
\begin{align}
    \mathcal{I}(X;H)-\beta \cdot \mathcal{I}(H;Y)
\end{align}
where the positive constant $\beta$ (Lagrange multiplier) acts as a trade-off parameter between the complexity of the representation $\mathcal{I}(X;H)$ and the amount of preserved relevant information $\mathcal{I}(H;Y)$. The intuition here is that we would like $H$ to be as much of a compressed version of $X$ as possible (by minimizing $\mathcal{I}(X;H)$), while being as predictive as possible about the outcome $Y$ (by maximizing $\mathcal{I}(H;Y)$). How much we favour prediction over compression is quantified by the trade-off parameter $\beta>0$. 

\paragraph{Application to RNNs.} If we consider an RNN with inputs $x_1,\ldots,x_t$, hidden states $h_1,\ldots,h_t$ and an output $y_t$ at time step $t$, then one way of applying IB in this context is to use $h_t$ as the representation $H$. In other words we would like to minimize $\mathcal{I}(x_{1:t};h_t)-\beta \cdot \mathcal{I}(h_t;y_t)$, where $x_{1:t} = (x_1,\ldots,x_t)$. Here $\mathcal{I}(x_{1:t};h_t)$ can then be estimated via the following corollary of the chain rule of mutual information (proposition \ref{CRMI}).
\begin{cor} \label{IM_prop}
Let $x_1,\ldots,x_T$ be the random variables corresponding to the input sequence of a recurrent neural network. Then we have for all positive integers $t$,$k$, and $s$
\begin{align}
    \mathcal{I}(h_{t+k};x_{s:t}) &= \sum_{i=s}^t \mathbb{E}_{x_{s:i-1}}\left[\mathbb{E}_{x_i| x_{s:i-1}}\left[D_{KL}(p(h_{t+k}|x_{s:i})|| p(h_{t+k}|x_{s:i-1}))\right]\right]
\end{align}
where $x_{s:i} = (x_s,\ldots,x_i)$ if $i\geq s$ and $x_{s:i} = \emptyset$ otherwise.
\end{cor}
\begin{proof}
For any positive integers $t$, $k$ and $s$ with $t\geq s$, we get, using the chain rule of mutual information (proposition \ref{CRMI}), 
\begin{align}
    \mathcal{I}(h_{t+k}; x_{s:t}) &= \mathcal{I}(h_{t+k}; (x_t,\ldots,x_s))\\
    &=\sum_{i=s}^t \mathcal{I}(h_{t+k}; x_i | (x_{i-1},\ldots,x_s))\\
    &=\sum_{i=s}^t \mathcal{I}(h_{t+k}; x_i | x_{s:i-1})
\end{align}
Finally using the equation \ref{cond_IM} from above, we have 
\begin{align}
    \mathcal{I}(h_{t+k}; x_i | x_{s:i-1}) &=\mathbb{E}_{x_{s:i-1}}\left[\mathbb{E}_{x_i| x_{s:i-1}}\left[D_{KL}(p(h_{t+k}|x_{s:i})|| p(h_{t+k}|x_{s:i-1}))\right]\right]
\end{align}
proving proposition \ref{IM_prop}.
\end{proof}

\section{Impact of adaptation on performance : further results}\label{appendix:p_results}

In this appendix section, we show additional results on the performance through learning.

\begin{figure}[h]
    \centering
    \includegraphics[width=0.85\textwidth]{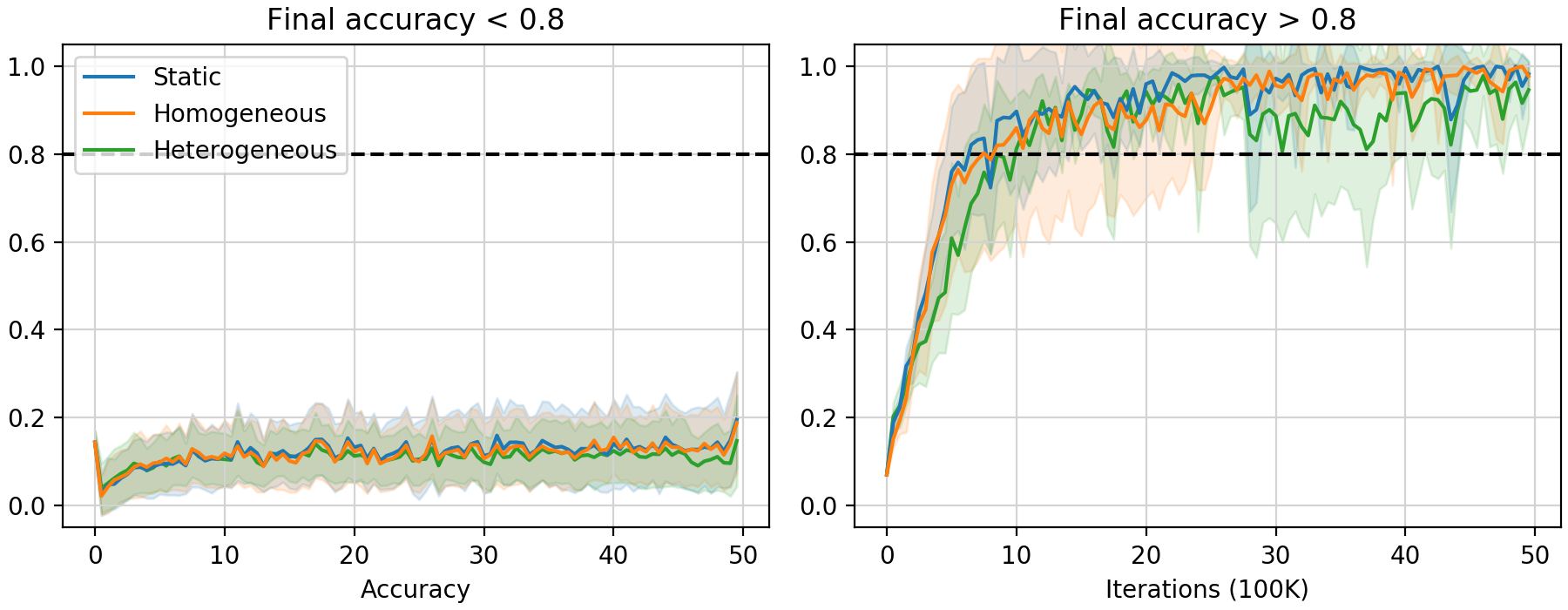}
    \caption{Accuracy on the copy task, for each scenario over training. The models are divided according to their final performance, plotted on the \textbf{left} if they do not exceed 0.8 in accuracy, and on the \textbf{right} if they do. Line corresponds to the average over the models (per scenario) and shaded regions indicated one standard deviation from the mean. }
    \label{fig:copy_performance_epochs}
\end{figure}

\begin{figure}[t]
    \centering
    \includegraphics[width=0.9\textwidth]{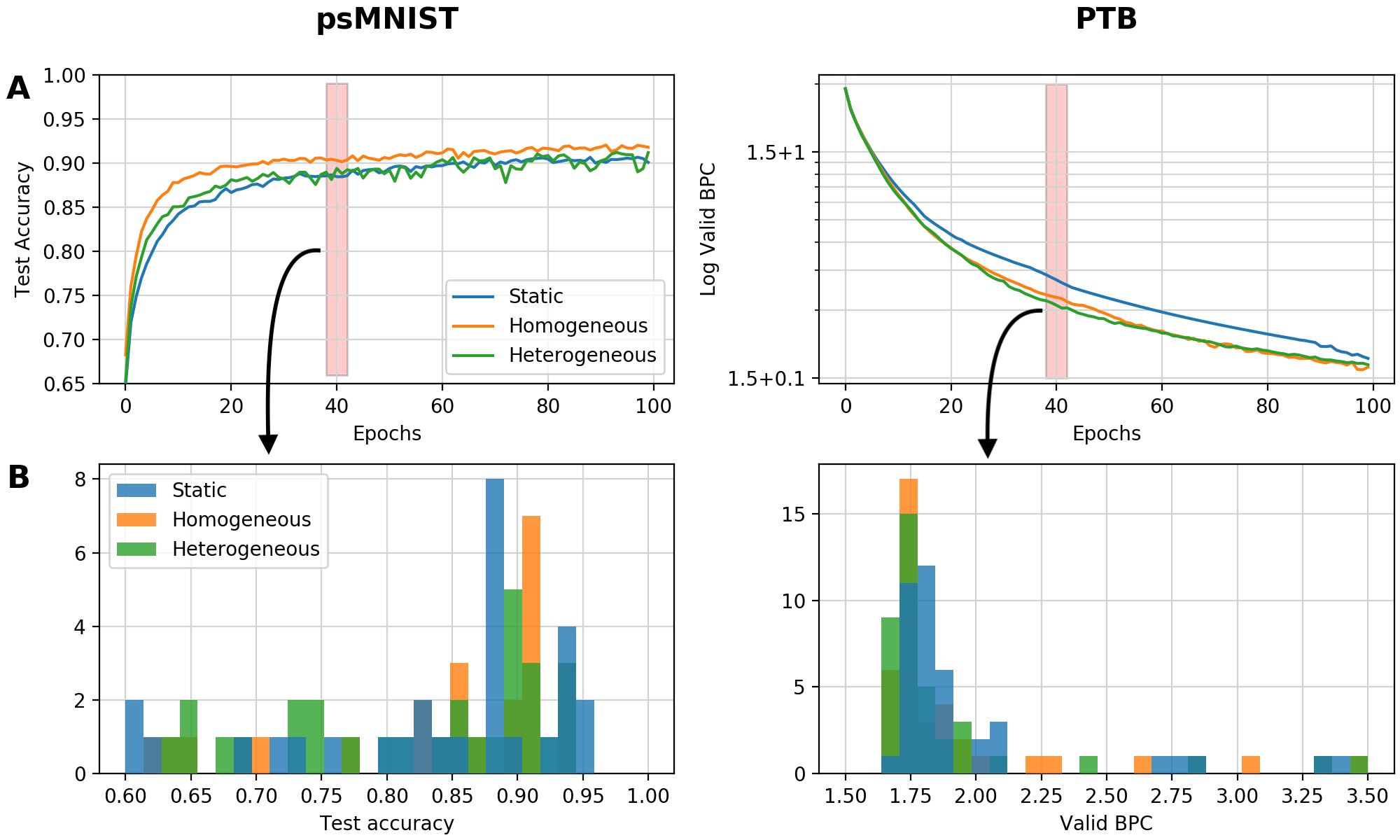}
    \caption{Accuracies and distributions for the psMNIST and PTB tasks, for each scenario over training. For both rows, legend inlet on the left applies to the whole row.  \textbf{A} Modal performance over training per scenario, for each task. Red rectangle acts as a visual aid for next row. \textbf{B.} Histogram of performances for each task and scenario at epoch 40. We point out the skewed nature of the distributions. Elements classified into bins are defined by a given combination of $n$, $s$ and adaptation scenario. }
    \label{fig:performance_epochs}
\end{figure}

For copy, we plot in Figure \ref{fig:copy_performance_epochs} the accuracy over training for each scenario. This reinforces what is mentioned in the main text, how models in the homogeneous adaptation setting perform similarly to the static setting. Those in the heterogeneous adaptation setting perform slightly worse, although the best performing still achieves an accuracy of 1.0. Also to reiterate what is mentioned in the main text, only models with saturation levels on the $s=0$ axis attain a decent accuracy. There is a clear distinction in terms of performance: either performance remains poor (which can be seen by the low mean and the standard deviation in Figure \ref{fig:copy_performance_epochs}-left), or the models perform well (Figure \ref{fig:copy_performance_epochs}-right).

For the two other sequential tasks (psMNIST, PTB), we plot in Figure \ref{fig:performance_epochs}A the modal performance over epochs. We found the distribution of performances over the parameters $(n,s)$ to be strongly skewed in the direction of the desired performances (positively skewed for PTB, negatively for psMNIST). For this reason, we opted for the distribution mode as a means of summarizing each scenario for a given epoch. To illustrate the distribution for a specific epoch (in the case 40) we refer to \ref{fig:performance_epochs}B. Neither an average nor a median would justly represent the data.

\section{Impact of the shape parameters on encoder efficiency}\label{appendix:MI_calculation}

This section explains how the mutual information experiments were conducted. We refer to Appendix \S\ref{MI_intro} above for a theoretical primer. 

As previously mentioned, the RNN used can be thought of as an encoder $p_{\boldsymbol{\theta}}(h_t|x_t)$ and a decoder $p_{\boldsymbol{\theta}}(\hat{y}_t|h_t)$ where $x_t, h_t, \hat{y}_t$ are realizations of the random variables $X_{\leq t}, H_t, \hat{Y}_t$, respectively the inputs, the hidden (or latent) state and the prediction made (or output), each depending on the time step $t$.
The decoder $p_{\boldsymbol{\theta}}(\hat{y}_t|h_t)$ is of little interest here since the RNN output is simply a linear readout of the internal hidden state. The encoder $p_{\boldsymbol{\theta}}(h_t|x_t)$ is however directly affected by the shape parameters $(n,s)$ and we are interested in quantifying the effect of these parameters on the encoder performance in the context of the PTB task. For these reasons, while we rely on the Information Bottleneck (see above \S\ref{appendix:ss:IB}) for an intuition of the desired interplay between encoding and decoding, we only compute the mutual information between the distributions of the encoder.

In order to isolate the effect of the shape parameters on the mutual information between the input $X_{\leq t}$ and the hidden state $H_t$ variables, we conducted our experiments on randomly initialized \cite{henaff} networks with manually specified values for $(n,s)$ and no subsequent training. Since the input is either a discrete or a continuous random variable (depending on the task) and the hidden state is continuous we used the mutual information estimator for discrete-continuous mixtures proposed in \cite{DBLP:conf/nips/GaoKOV17}. This estimator was shown to be better than previous state-of-the-art estimators for situations where "some variables are discrete, some continuous, and others are a mixture between continuous and discrete components".

The goal here being to isolate as much as possible the encoder from the internal dynamics, we chose to only consider the present input $x_t$ and the present hidden state $h_t$ to compute the mutual information estimate $\hat{\mathcal{I}}_{(n,s)}(h_t;x_t)$ that is plotted in Figure \ref{fig:signals_fig}B.
For the copy and PTB tasks, $x_t$ is a discrete random variable representing a digit (values 0 to 9) for copy and the index of an alphabet character for PTB (values 0 to 25). For the psMNIST task $x_t$ is continuous and represents the color of a pixel in gray scale (values 0 to 1). In all cases, $h_t$ is continuous and contains the activation values of the $N$ hidden neurons.

\section{Transfer learning experiments}\label{appendix:transfer}

For our transfer learning experiments, we alter the data in the MNIST dataset by imposing an anti-clockwise rotation of 45° over the entire image for each figure.
Only the networks which originally achieved an accuracy of over 94\% on the psMNIST classification task were used as initializations for the transfer learning task. Mean accuracy and standard deviation of the networks during retraining is plotted in Figure \ref{fig:appendix:shapeonly_performance} as well as the accuracy throughout retraining for the top performing network at the end of the retraining.
\begin{figure}[H]
    \centering
    \includegraphics[width=0.6\textwidth]{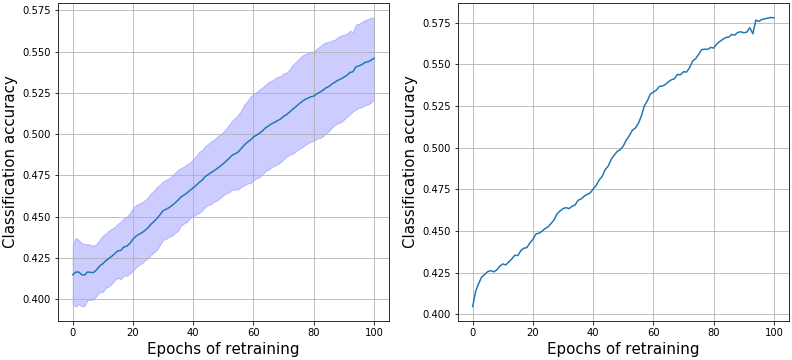}
    \caption{\textbf{Left:} Mean and standard deviation for accuracy during retraining of the top 5 performing networks. \textbf{Right:} Best performing network after retraining on the new task.}
    \label{fig:appendix:shapeonly_performance}
    \vspace{-10pt}
\end{figure}

In the article, the shape parameters trajectories during retraining for the transfer task were plotted for only one network for clarity. In Figure \ref{fig:appendix:shape_trajectories} we show the same plot but for the 4 remaining tested networks. What we can observe is that the expansion in parameter space mentioned in the article is present in all subplots. It is a generally observable phenomenon which supports the idea of the article that a diversification in nonlinearity shapes is needed to adapt to the change in task.

\begin{figure}[H]
\centering
\includegraphics[width=\textwidth]{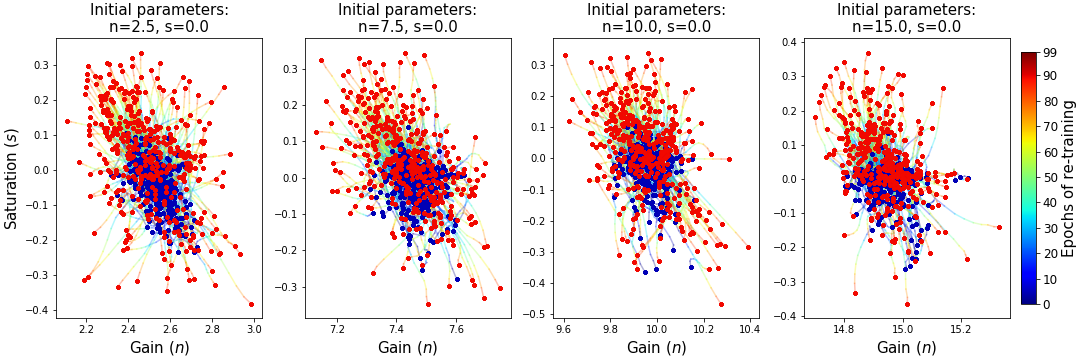}
  \caption{Trajectories of the shape parameters during retraining on the modified MNIST images (Figure \ref{fig:shape_trajectories} of the article) for 4 other network initializations.}
  \label{fig:appendix:shape_trajectories}
  \vspace{-10pt}
\end{figure}

\bibliographyNew{main, main_2}
\bibliographystyleNew{unsrt}

\end{document}